\documentclass[twoside,11pt]{article}

\usepackage{jmlr2e}

\usepackage{hyperref}
\usepackage{url}
\usepackage{subcaption}
\usepackage[utf8]{inputenc} 
\usepackage[T1]{fontenc}    
\usepackage{url}            
\usepackage{xcolor} 
\usepackage{amsfonts}       
\usepackage{nicefrac}       
\usepackage{float}
\usepackage{wrapfig}
\usepackage[ruled,vlined,linesnumbered]{algorithm2e}
\usepackage{algpseudocode}
\usepackage{color, colortbl}
\usepackage{enumitem}
\usepackage{caption}
\usepackage{todonotes}
\usepackage{mdframed}
\usepackage[skins]{tcolorbox}
\definecolor{Gray}{gray}{0.9}
\definecolor{LightCyan}{rgb}{0.88,1,1}
\usepackage{subcaption}
\usepackage{graphicx}      
\usepackage{natbib}        

\usepackage{microtype}
\usepackage{booktabs} 

\usepackage{hyperref}

\newcommand{\muc}[2]{\multicolumn{#1}{c}{#2}}
\usepackage{booktabs}       
\usepackage{array}          
\usepackage{multirow}       
\usepackage{graphicx}       
\usepackage{xcolor}         
\usepackage{caption}        

\usepackage{mathtools}
\usepackage{amsthm}

\newtheorem{theorem}{Theorem}
\newtheorem{lemma}[theorem]{Lemma} 
\newtheorem{proposition}[theorem]{Proposition} 
\newtheorem{remark}[theorem]{Remark}

\newtheorem{definition}[theorem]{Definition}

\definecolor{colorone}{rgb}{0.00,0.45,0.74}
\definecolor{colortwo}{rgb}{0.85,0.33,0.10}
\definecolor{colorthree}{rgb}{0.49,0.18,0.56}
\definecolor{colorfour}{rgb}{0.93,0.69,0.13}
\definecolor{colorfive}{rgb}{0.47,0.67,0.19}
\definecolor{colorsix}{rgb}{0.30,0.75,0.93}

\usepackage{lastpage}

\firstpageno{1}

\newcommand{\statee}{s}

\newcommand{\reals}{\mathbb{R}}

\newcommand{\gaussian}{\mathcal{N}}

\newcommand{\horizon}{\mathrm{K}}
\newcommand{\timeid}{k}
\newcommand{\init}{\mathcal{I}}

\newcommand{\traj}{\sigma}
\newcommand{\trajsim}{\traj^{\mathsf{sim}}}
\newcommand{\ressim}{\rho}
\newcommand{\resreal}{\rho}

\newcommand{\errsim}[1]{R^{#1}}
\newcommand{\errreal}[1]{R^{#1}}

\newcommand{\trajsimseg}[1]{\traj^{\mathsf{sim} , #1}}
\newcommand{\PE}{\mathsf{PE}}
\newcommand{\PEreal}{\mathsf{PE}}

\newcommand{\PEsimseg}[1]{\mathsf{PE}^{#1}}
\newcommand{\eigvecseg}[1]{\mathsf{V}^{ #1}}
\newcommand{\trajreal}{\traj^{\mathsf{real}}}

\newcommand{\relu}{\mathsf{ReLU}}

\newcommand{\overallf}{\mathcal{F}}

\newcommand{\transpose}[1]{{#1}^{\top}}

\newcommand{\dist}{\trajdist^{\mathsf{real}}}
\newcommand{\distzero}{\trajdist^{\mathsf{sim}}}
\newcommand{\distR}{\trajdistR^{\mathsf{real}}}
\newcommand{\distzeroR}{\trajdistR^{\mathsf{sim}}}

\newcommand{\trajdataset}{\mathcal{T}}
\newcommand{\traindataset}{\trajdataset^{\mathsf{trn}}}
\newcommand{\tv}{\mathsf{TV}}

\newcommand{\calibdataset}{\mathcal{R}^{\mathsf{calib}}}

\newcommand{\mypara}[1]{\vspace{0.3em} \noindent{\bf #1}.}

\def\relu{\mathrm{ReLU}}

\newcommand{\navid}[1]{\textcolor{black}{#1}}
\newcommand{\navidd}[1]{\textcolor{black}{#1}}
\newcommand{\navidg}[1]{\textcolor{black}{#1}}

\newcommand{\distinit}{\mathcal{W}}
\newcommand{\states}{\mathcal{S}}
\newcommand{\Statee}{S}
\newcommand{\trajdist}{\mathcal{D}_{\Statee,\horizon}}
\newcommand{\trajdistR}{\mathcal{J}_{\Statee,\horizon}}

\begin{document}

\title{PCA-DDReach: Efficient Statistical Reachability Analysis of Stochastic Dynamical Systems via Principal Component Analysis}
\author{\name Navid Hashemi     \email navidhas@usc.edu \\
       \addr University of Southern California, Los Angeles, California, United States\\
       \AND
       \name Lars Lindemann        \email llindema@usc.edu \\
       \addr University of Southern California, Los Angeles, California, United States\\
       \AND
       \name Jyotirmoy Deshmukh \email jdeshmuk@usc.edu\\
       \addr University of Southern California, Los Angeles, California, United States\\
       }

\maketitle

\begin{abstract}
This study presents a scalable data-driven algorithm designed to efficiently address the challenging problem of reachability analysis. Analysis of cyber-physical systems (CPS) relies typically on parametric physical models of dynamical systems. However, identifying parametric physical models for complex CPS is challenging due to their complexity, uncertainty, and variability, often rendering them as black-box oracles. As an alternative, one can treat these complex systems as black-box models and use trajectory data sampled from the system (e.g., from high-fidelity simulators or the real system) along with machine learning techniques to learn models that approximate the underlying dynamics. However, these machine learning models can be inaccurate, highlighting the need for statistical tools to quantify errors. Recent advancements in the field include the incorporation of statistical uncertainty quantification tools such as conformal inference (CI) that can provide probabilistic reachable sets with provable guarantees. Recent work has even highlighted the ability of these tools to address the case where the distribution of trajectories sampled during training time are different from the distribution of trajectories encountered during deployment time.  However, accounting for such distribution shifts typically results in more conservative guarantees. This is undesirable in practice and motivates us to present techniques that can reduce conservatism. Here, we propose a new approach that reduces conservatism and improves scalability by combining conformal inference with Principal Component Analysis (PCA).  We show the effectiveness of our technique on various case studies, including a 12-dimensional quadcopter and a 27-dimensional hybrid system known as the powertrain.
\end{abstract}


\section{Introduction}
\label{sec:intro}
\navid{System} verification tools are crucial for ensuring correctness prior to testing, implementation, or deployment, particularly in expensive, high-risk, or safety-critical systems \cite{zhang2023reachability,schilling2022verification,komendera2012intelligent}. In real-world implementations, we face two significant challenges for \navid{system} verification: (1) assuming access to an underlying model for the system is \navid{often prohibitive}, and (2) the presence of noise and uncertainty results in \navid{stochastic systems}, requiring us to perform \navidd{statistical} verification \navidd{to obtain} probabilistic \navidd{safety certificates}. Data-driven statistical \navidd{reachability analysis} is a well-established tool for \navidd{statistical} verification, \navidd{and has been applied} to safety-critical problems across various domains, \navidd{including, autonomy, medical imaging and CPS. } \cite{devonport2021data,fisac2018general,dryvr,hashemi2024statistical}. 

In the context of data-driven statistical reachability analysis, given a user-specified threshold $\delta \in (0,1)$, the goal is to produce a set that \navidd{ensures that} any trajectory during deployment lies within \navidd{this set} with a probability of \navidd{no less than} $\delta$. \navidd{As explained in the "\textit{Related Work}" section, our work is motivated by the method proposed in \cite{hashemi2024statistical}. This methodology involves three }main steps: (1) learning a deterministic surrogate model from sampled trajectories, (2) performing reachability analysis on the surrogate model, and (3) using robust conformal inference  \cite{cauchois2020robust} to calculate an inflating hypercube. This inflating hypercube quantifies the necessary expansion of the surrogate model's reachable set to provide provable probabilistic guarantees on the flowpipe that is obtained for the black-box oracle. 

Nonetheless, the utilized technique in \cite{hashemi2024statistical} to integrate robust conformal inference, \navid{can be adjusted to reduce the level of conservatism}. Our first contribution in this paper is to address this issue, by combining robust conformal inference with Principal Component Analysis (PCA), and we show this adjustment results in tighter inflating hypercubes and probabilistic reachable sets. \navid{
In \cite{hashemi2024statistical}, the authors assume a single model that maps the initial state to the entire trajectory\footnote{This approach is motivated by the fact that training a one-step model and iterating it over time leads to the well-known issue of cumulative errors}. However, this results in a large model, which limits the scalability of reachability analysis for extended horizons. Training such a large model becomes inefficient. Furthermore, the model’s large size makes accurate methods for surrogate reachability infeasible, necessitating the use of conservative over approximation techniques.}
 In response, our second contribution here is to address these challenges by presenting a new training strategy that effectively resolves for all of these listed scalability issues. \navid{We propose training a set of independent small models, each mapping the initial state to a separate sub-partition of the trajectory. This approach eliminates the need for iterating a model over time while maintaining smaller models that are more suitable for efficient training and efficient surrogate reachability analysis.}

\mypara{Related Work}  \navidd{In the context of data-driven reachability analysis, typically, to obtain the reachable set, a dataset of system trajectories needs to be available, e.g., from a black box simulator.} In \cite{devonport2021data}, the authors use Christoffel functions\footnote{See \cite{lasserre2019empirical,marx2021semi} for more details about the Christoffel functions.} to \navidd{learn the model of the system from such a } dataset and generate probabilistic \navidd{reachable} sets. This methodology has been extended in \cite{tebjou2023data} by incorporating conformal inference \cite{vovk2012conditional} \navidd{to improve its data efficiency}. In \cite{devonport2020data}, a Gaussian process-based classifier is employed to learn a model \navidd{that distinguishes} between reachable and unreachable states to approximate the \navidd{reachable} set. \navid{In \cite{devonport2020estimating,dietrich2024nonconvex} \navidd{scenario optimization is used} to generate probabilistic reachable sets.} The method in \cite{fisac2018general} assumes partial knowledge of the \navidd{system} and leverages the dataset to perform statistical reachability analysis. Similarly, the work presented in \cite{dryvr} uses the dataset to \navidd{learn an exponential discrepancy function that estimates trajectory's sensitivity to uncertainty, enabling the computation of probabilistic reachable sets}. Of our particular interest is the method proposed by \cite{ hashemi2024statistical}, which suggests learning a \navidd{system} model on this dataset via $\text{ReLU}$ neural networks \navidd{and then conduct statistical reachability analysis via conformal inference.}  \navidd{By leveraging the ability of neural networks to model complex, high-dimensional relationships alongside the data efficiency of conformal inference, this approach establishes an efficient and structured framework for statistical reachability analysis.} 
Additionally, the probabilistic guarantees proposed by \cite{hashemi2024statistical} remain valid even when there is a distribution shift between the training and deployment environments. \navidd{This is the main reason we focus on extending this work instead of building on other existing methodologies.}

\navidd{Conformal inference (CI) is a data-efficient method for formally providing guarantees on the $\delta$-quantile of distributions. This method involves sampling an i.i.d. scalar dataset, sorting the samples in ascending order, and demonstrating that one of the sorted samples represents the $\delta$-quantile.} The integration of CI with formal verification techniques has recently received noticeable interest, that is primarily due to its accuracy and level of scalability. For instance, 
\cite{bortolussi2019neural} merges CI with neural state classifiers to develop a stochastic runtime verification algorithm. \cite{lindemann2023safe,zecchin2024forking} employ CI to guarantee safety in MPC control using a trained model. \cite{tonkens2023scalable}  applies CI for planning with probabilistic safety guarantees, and \cite{hashemi2024statistical} integrates CI with existing neural network reachability techniques and provides a scalable reachability analysis on stochastic systems, see \cite{lindemann2024formal} for a recent survey article.

\mypara{Notation} We use bold letters to represent vectors and vector-valued functions, while \navid{caligraphic} letters  denote sets and distributions. The set $\left\{1,2,\ldots, n \right\}$ is denoted as $[n]$. The Minkowski sum is indicated by $\oplus$. We use $x \sim \mathcal{X}$ to denote that the random variable $x$ is drawn from the distribution $\mathcal{X}$. We present the structure of a feedforward neural network (FFNN) with $\ell$  hidden layers as  an array $[n_0,n_1,\ldots n_{\ell+1}]$, where $n_0$ denotes the number of inputs, $n_{\ell+1}$  is the number of outputs, and $n_i , i\in [\ell]$ denotes the width of the $i$-th hidden layer. We denote $e_i\in\mathbb{R}^n$ as the $i$-th base vector of $\mathbb{R}^n$. We also denote $\lceil x \rceil$ as the smallest integer greater than $x \in \mathbb{R}$.

\section{Preliminaries}
\label{sec:prelim}
\subsection{Stochastic Dynamical Systems} 
\vspace{-2mm}
Consider a set of random vectors $\Statee_0, \ldots, \Statee_\horizon \in \states$ indexed at times $0, \ldots, \horizon$ and with state space $\states\subseteq\mathbb{R}^n$. A realization of this stochastic process is a sequence of values $\statee_1, \ldots, \statee_\horizon$, denoted as system trajectory  $\trajreal_{\statee_0}$. The joint distribution over $\Statee_1, \ldots, \Statee_\horizon$ is the trajectory distribution $\dist$, while the marginal distribution of $\Statee_0$ is known as the initial state distribution $\distinit$. It is assumed that $\distinit$ has support over a compact set of initial states $\init$, implying $\Pr[\statee_0 \notin \init] = 0$.

\mypara{Training and Deployment Environments} In the training environment, we pre-record or simulate datasets to conduct reachability analysis. Conversely, the deployment environment refers to the real world where we apply our reachable sets. There is typically a difference between the distribution of trajectories in the training and deployment environments. We refer to this difference as distribution shift. In this paper, we assume that for a predefined distribution on initial states $\statee_0 \sim \mathcal{W}$, the real-world trajectories $\trajreal_{\statee_0}$ are sampled from $\trajreal_{\statee_0} \sim \dist$, whereas the simulated trajectories $\trajsim_{\statee_0}$ are sampled from $\trajsim_{\statee_0} \sim \distzero$. 

\vspace{-2mm}
\subsection{Surrogate Model: Reachability \& Error Analysis}
 \vspace{-2mm}
 A surrogate model $\overallf: \init \times \Theta \to \states^\horizon$, with trainable parameters $\theta \in \Theta$, can be trained by sampling $\horizon$-step trajectories $\trajsim_{\statee_0} \sim \distzero$ from the simulator to predict the trajectory $\trajsim_{\statee_{0}} \in \states^\horizon$ given its initial state $\statee_0 \in \init$.  We call this dataset $\traindataset$, and we denote the predicted trajectory by,
\begin{equation}\label{eq:surrogate}
\bar{\traj}_{\statee_{0}} = \overallf(\statee_{0}\ ; \theta),\ \text{where},\ \  \overallf(\statee_{0}\ ; \theta) = 
\left[\mathsf{F}^1(\statee_0), \ldots, \mathsf{F}^n(\statee_0),  \ldots, 
\mathsf{F}^{(\horizon-1)n+1}(\statee_0), \ldots, \mathsf{F}^{n\horizon}(\statee_0)\right]^\top 
\end{equation}
\navid{where}, $\mathsf{F}^{(\timeid-1)n+\ell}(\statee_0)$ is the $\ell^{th}$ state component at the $\timeid^{th}$ time-step in the trajectory\footnote{Here, the dimension and time steps are stacked into a single vector.}. Let $e_\ell\in\reals^n$ denote the $\ell$-th basis vector of $\reals^n$. For a trajectory $\statee_1, \ldots \statee_\horizon$, and $\statee_0 \sim \mathcal{W}$, for $j =(\timeid-1)n+\ell$, we define the prediction errors as,
\begin{equation}
\label{eq:compwise_residual}
R^j =
e_{\ell}^\top \statee_\timeid - \mathsf{F}^j(\statee_0), \quad \timeid \in[\horizon], \ell\in[n].
\end{equation}

In this paper, we \navidd{introduce the residual $\rho:\mathbb{R}^{n\horizon} \to \mathbb{R}_{\ge 0}$ as a function} of the prediction errors $R^j$, where $j \in [n\horizon]$. In this section, we present a previously proposed example of such a \navid{function and later introduce an adjustment to reduce the conservatism in probabilistic reachability}.
\begin{definition}[Simulation \& Real Residual Distribution]
If the residual is generated by $\trajsim_{\statee_0}\sim \distzero$, we denote the residual distribution as $\ressim \sim \distzeroR$ where $\distzeroR$ is the simulation residual distribution. Conversely, if the residual is generated by $\trajreal_{\statee_0}\sim \dist$, we denote it as $\resreal \sim \distR$ where $\distR$ is the real residual distribution.
\end{definition}
We also utilize total variation, \cite{takezawa2005introduction} as a metric to quantify their distribution shift, $\tau\geq 0$. In other word, $\tau = \tv(\distzeroR, \distR)$, where $\tv$ refers to the total variation.

\mypara{Surrogate Flowpipe and Star-Set}  The surrogate flowpipe $\bar{X}\subset \reals^{n\horizon}$ is defined as a superset of the image of $\overallf(\init\ ;\theta)$. Formally, for all  $\statee_0\in \init$, we need that $\overallf(\statee_0\ ;\theta) \in \bar{X}$. Due to the recent achievements in verifying neural networks with ReLU activations, we limit ourselves to the choice of $\relu$ neural networks as our surrogate models, and we rely on the NNV toolbox from \cite{tran2020nnv} to compute the surrogate flowpipe.  Although other activation functions can be used, we anticipate more conservative results if non-ReLU activation functions are utilized. \navidg{We choose to use NNV in our analysis here, because it can yield accurate reachability results in settings where neural networks with $\relu$ activation function are used. } The approach in \cite{tran2020nnv} employs star-sets (an extension of zonotopes) to represent the reachable set and utilizes two main methods: (1) the exact-star method, which performs precise but slow computations, and (2) the approx-star method, which is faster but it is more conservative.
\begin{definition}[Star set \cite{bak2017simulation}]\label{def:star} A \textit{star set} $Y\subset \mathbb{R}^d$ is a tuple $\langle c, V, P \rangle$ where $c \in \mathbb{R}^d$ is the center, $V = \{v_1, v_2, \ldots, v_m \}$ is a set of $m$ vectors in $\mathbb{R}^d$ called \textit{basis vectors}, and $P : \mathbb{R}^m \rightarrow \{\top, \bot\}$ is a predicate. The basis vectors are arranged to form the star's $d \times m$ basis matrix. \navid{Given variables $\mu_\ell \in \mathbb{R}, \ell = 1,\ldots,m$}, the set of states represented by the star is given as:
\begin{equation}
 Y  = \left\{ y \mid y = c + \sum_{\ell=1}^{m} (\mu_\ell v_\ell) \text{ s.t. } P(\mu_1, \ldots, \mu_m) = \top \right\}.     
\end{equation}
\end{definition}


\vspace{-2mm}
\subsection{Conformal Inference \& Probabilistic Reachability }
\vspace{-2mm}
\navid{A key step toward probabilistic reachability is to provide a provable $\delta$-quantile for the residual}. Let $\ressim_1 < \ressim_2 < \ldots < \ressim_L$ represent $L$ different i.i.d. residuals sampled from $\distzeroR$, and sorted \navid{in ascending order}. Given a confidence probability,  $\delta \in (0,1) $, a provable tight upper bound for the $\delta$-quantile of the residuals $\resreal \sim \distR$ is computable from samples $\ressim_i\sim \distzeroR , i\in[L]$, using robust conformal inference proposed in \cite{cauchois2020robust} that is an extension of CI, proposed in \cite{vovk2012conditional}.

\navid{The theory of conformal inference  states that for a new sample \(\ressim \sim \distzeroR\), the rank \(\ell := \lceil (L+1)\delta\rceil \le L\) satisfies $\Pr[ \ressim < \ressim_\ell] \geq \delta.$ This implies that \(\ressim_\ell\) serves as a provable upper bound for the \(\delta\)-quantile of \(\distzeroR\). However, this result does not extend to another residual \(\resreal \sim \distR\), which is drawn from a different distribution. To address this distribution shift, the theory of robust conformal inference introduces an adjustment to conformal inference. It establishes that for any random variable \(\resreal \sim \distR\) satisfying \(\tv(\distR, \distzeroR) \leq \tau\) with a threshold \(\tau > 0\), we have $\Pr[\resreal < \ressim_{\ell^*}] > \delta$, where  
\begin{equation}\label{eq:robustrank}
\ell^*:=\lceil (L+1)(1+1/L)(\delta+\tau)\rceil,\ \ \ell^* \leq L.
\end{equation}  
Thus, \(\ressim_{\ell^*}\) serves as an upper bound for the \(\delta\)-quantile of \(\distR\).}


\mypara{Inflating Hypercube} In reachability analysis, the main purpose for the definition of the residual is to achieve a \navid{bounding region} that will cover the random sequence of prediction errors, $\PE = [R^1, R^2, \ldots, R^{n\horizon}]$ with a  confidence $\delta\in (0,1)$. In this case, as suggested by \cite{cleaveland2023conformal}, the $\max()$ operator over the absolute value of all errors is a suitable choice. The authors in \cite{hashemi2024statistical}, for some positive constants $\alpha_j, j\in[n\horizon]$ (details on the choice of $\alpha_j$ can  be found in \cite{hashemi2024statistical}), define   the residual
\begin{equation}\label{eq:Rmax}
\rho:= R = \max( \alpha_1 |R^1|, \alpha_2 |R^2|, \ldots, \alpha_{n\horizon}|R^{n\horizon}|),
\end{equation}
and \navidd{show} that such a bounding region is achievable by computing an upper bound for the $\delta$-quantile of residual, $R$. In other words, assuming $R^*$ as the mentioned upper bound, we have,
\begin{equation}\label{eq:Pstar}
\Pr[R<R^*] \geq \delta \!\! \navid{\iff} \!\!\Pr[P^*(R^1,\ldots,R^{n\horizon})\! = \!\top] \geq \delta,\ \   P^*(R^1,\ldots,R^{n\horizon})\!\! = \!\!\! \bigwedge_{j=1}^{n\horizon}\! (|R^j| \! < \! \frac{R^*}{\alpha_j})
\end{equation}
where $R^*$ is efficiently obtainable via robust conformal inference. \navidd{As defined in \eqref{eq:Pstar}, the predicate $P^*$ implies} that for every component $R^j, j=1,\ldots,n\horizon$ of the vector $\PE$ we have $-R^*/\alpha_j\leq R^j \leq R^*/\alpha_j$. This describes a hypercube which can be formulated as the following star set:
\begin{equation}
\delta X = \langle\  0_{n\horizon \times 1},\  I_{n\horizon},\  P^*(R^1,\ldots,R^{n\horizon}) \ \rangle \subset \mathbb{R}^{n\horizon}.
\end{equation}
Since $\Pr[P^* = \top] \geq \delta$, this star set serves as such a bounding region for $\PE$. We refer to this bounding region as inflating hypercube.

\mypara{$\delta$-Confident Flowpipe \& Probabilistic Reachability} For a given confidence probability $\delta\in (0,\ 1)$, and $\statee_0 \sim \mathcal{W}$, we say that $X \subseteq \reals^{n\horizon}$ is a $\delta$-confident flowpipe if for any random trajectory $\trajreal_{\statee_0} \sim \dist$, we have $\Pr[\trajreal_{\statee_0} \in X ] \geq \delta$. \navidd{In this paper, our ultimate goal is to propose a $\delta$-confident flowpipe. To compute such a flowpipe, the authors in \cite{hashemi2024statistical}} suggest simulating a set of trajectories, $\trajsim_{\statee_0} \sim \distzero, \statee_0 \sim \mathcal{W}$ and training a $\relu$ NN surrogate model $\overallf(\statee_0 ; \theta)$ on this dataset. This model will be utilized to compute for its surrogate reachset, $\bar{X}\subset \mathbb{R}^{n\horizon}$. They also sample a new set of trajectories $\trajsim_{\statee_0} \sim \distzero, \statee_0 \sim \mathcal{W}$ for error analysis on $\overallf(\statee_0 ; \theta)$ through robust conformal inference to compute for another hypercube $\delta X \subset \mathbb{R}^{n\horizon}$, known as the inflating hypercube, that covers the prediction errors $R^j , j\in[n\horizon]$ for trajectories $\trajreal_{\statee_0} \sim \dist$, with a provable probabilistic guarantee, and finally they propose the following lemma to compute for the $\delta$-confident flowpipe on $\trajreal_{\statee_0} \sim \dist, \statee_0 \sim \mathcal{W}$. See \cite{hashemi2024statistical} for the proof.
\begin{lemma}\label{lem:inclusion_conformal}
Let $\bar{X}$ be a surrogate flowpipe of the surrogate model $\overallf$ for the set of initial conditions $\init$. Let $\PEreal:=\left[ \errreal{1} , \errreal{2}, \ldots , \errreal{n\horizon} \right] $ be the sequence of prediction errors for $\trajreal_{\statee_0} \sim \dist$, where $\statee_0 \sim \mathcal{W}$, and  let $\delta X$ be the inflating hypercube for $\PEreal$ such that $\Pr[\PEreal \in \delta X] > \delta$. Then the inflated reachset $X = \bar{X} \oplus \delta X $ is a $\delta$-confident flowpipe for $\trajreal_{\statee_0} \sim \dist$ where $\statee_0 \sim \mathcal{W}$.
\end{lemma}

\vspace{-2mm}
\subsection{Problem Definition}
\vspace{-2mm}
We are interested in computing  a $\delta$-confident flowpipe $X$ from a set of trajectories $\trajsim_{\statee_0}$ collected from $ \distzero$ so that $X$ is also valid  for all trajectories $\trajreal_{\statee_0} \sim \dist$ when the total variation between $\distR$ and $\distzeroR$ is less than $\tau>0$. 
\navidd{While we are motivated by the results in \cite{hashemi2024statistical} which propose a  solution to the stated problem, we note that their solution lacks scalability and accuracy that results in sometimes large levels of conservatism, i.e., the set $X$ is unnecessarily large.}

\navid{The primary sources of conservatism and inaccuracy in the methodology described in \cite{hashemi2024statistical} stem from the training process for the surrogate model $\overallf(\statee_0\ ; \theta)$ and the method used to compute the inflating hypercube $\delta X$. In the following sections, we address both issues and propose solutions to improve the accuracy and scalability of this approach for the reachability analysis.}


\section{Scalable and Accurate Data Driven Reachability Analysis}
\label{sec:DDReach}

In this section, we introduce two key adjustments to the methodology of \cite{hashemi2024statistical} to enhance scalability and reduce conservatism.

\vspace{-2mm}
\subsection{Improved Scalabilty and Accuracy for Training  Surrogate Models}
\label{sec:Training}
\vspace{-2mm}
\begin{figure*}
    \centering
    \includegraphics[width = 0.6\textwidth]{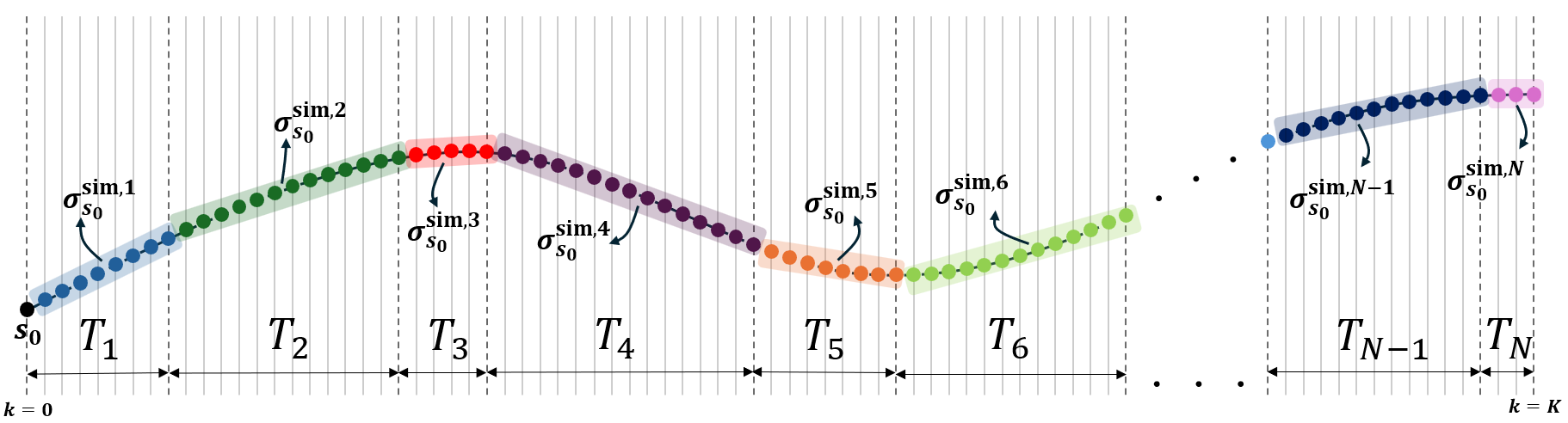}
    \caption{This figure shows the division of the trajectory into $N$ different segments $\trajsimseg{q}_{\statee_0}, q\in[N]$ }\label{fig:realization}
    \vspace{-5mm}
\end{figure*}

In this section, we introduce a new training strategy for the model $\overallf(\statee_0\ ;\theta)$ that avoids the scalability issues \navid{arising from the surrogate model's large size when handling long time horizons, as encountered in \cite{hashemi2024statistical}}. Figure \ref{fig:realization} illustrates a realization of a trajectory $\trajsim_{\statee_0} :=  \statee_1, \ldots, \statee_{\horizon}$ over the horizon $\horizon$. In this figure, we divide the time horizon into $N$ segments, each with length $T_q$, where $q \in [N]$. We denote each trajectory segment as $\trajsimseg{q}_{\statee_0},\ q \in [N]$, defined as:
\vspace{-3mm}
\begin{equation}
    \trajsimseg{q}_{\statee_0} := \statee_{t_q+1}, \statee_{t_q+2}, \ldots, \statee_{t_q+T_i}, \quad t_q = \sum_{\ell=1}^{q-1} T_\ell,\ t_1 = 0.
\end{equation}
\vspace{-3mm}\\
The key idea is to directly link each trajectory segment $\trajsimseg{q}_{\statee_0},\ q \in [N]$ to its initial state $\statee_0 \in \init$. Thus, we can train an independent model $\overallf_q(\statee_0\ ; \theta_q),\ q \in [N]$ for each segment, which predicts $\trajsimseg{q}_{\statee_0}$ directly based on the initial state $\statee_0$. This model is also used to compute surrogate flowpipes for the trajectory segments $\bar{X}_q, q \in [N]$, representing the image of set $\init$ through the model $\overallf_q(\init\ ; \theta_q)$.

Here are the reasons why this new training strategy for the trajectory $\trajsim_{\statee_0}$ resolves all the scalability issues, we listed for \cite{hashemi2024statistical}, in the Introduction section. \navid{First and foremost,} since all surrogate models $\overallf_q(\statee_0\ ; \theta_q), q \in [N]$ are directly connected to the initial state, we do not need \navid{to iterate} them \navid{sequentially} over the time horizon for prediction of states in  $\trajsimseg{q}_{\statee_0}, q\in[N]$, thus eliminating the problem of cumulative errors over the time horizon. \navid{Furthermore} in this setting, the size of the models $\overallf_q(\statee_0\ ; \theta_q), q \in [N]$ can be small. The small size of the models allows for efficient computation of surrogate flowpipes $\bar{X}_q := \overallf_q(\init\ ; \theta_q)$ for each segment via exact-star reachability analysis\footnote{However, if the set of initial states $\init$ is large and the partitioning of $\init$ is not scalable (high dimensional states), we remain limited to using approx star. Nevertheless, even for this case, the small size of the model significantly reduces the conservatism of approx star.}. 
\navid{Additionally, } the smaller models $\overallf_q(\statee_0\ ; \theta_q), q \in [N]$ enable efficient training of accurate models for each trajectory segment. \navidg{Furthermore, although we have to train more models using this technique,  we can train them in parallel as they are totally independent processes.}

\navid{Once the surrogate flowpipes \( \bar{X}_q = \langle \bar{c}_q, \bar{V}^q, \bar{P}_q \rangle \), for \( q \in [N] \), are obtained as star sets, the surrogate flowpipe for the entire trajectory forms another star set, \( \bar{X} = \langle \bar{c}, \bar{V}, \bar{P} \rangle \). This global surrogate flowpipe is constructed by concatenating all individual star sets \( \bar{X}_q \), for \( q \in [N] \), which implies: $\bar{c} = \left[ \bar{c}_1^\top , \ldots , \bar{c}_N^\top \right]^{\top}$, $\bar{V} = \mathbf{diag}\left(\bar{V}^1, \ldots, \bar{V}^N\right)$ and $\bar{P} = \bigwedge_{q=1}^N \bar{P}_q.$  }

\vspace{-2mm}
\subsection{Accurate Inflating Hypercubes  via Principal Component Analysis }
\label{sec:PCA}
\vspace{-2mm}
\begin{figure}
    \centering
    \includegraphics[width=0.45\linewidth]{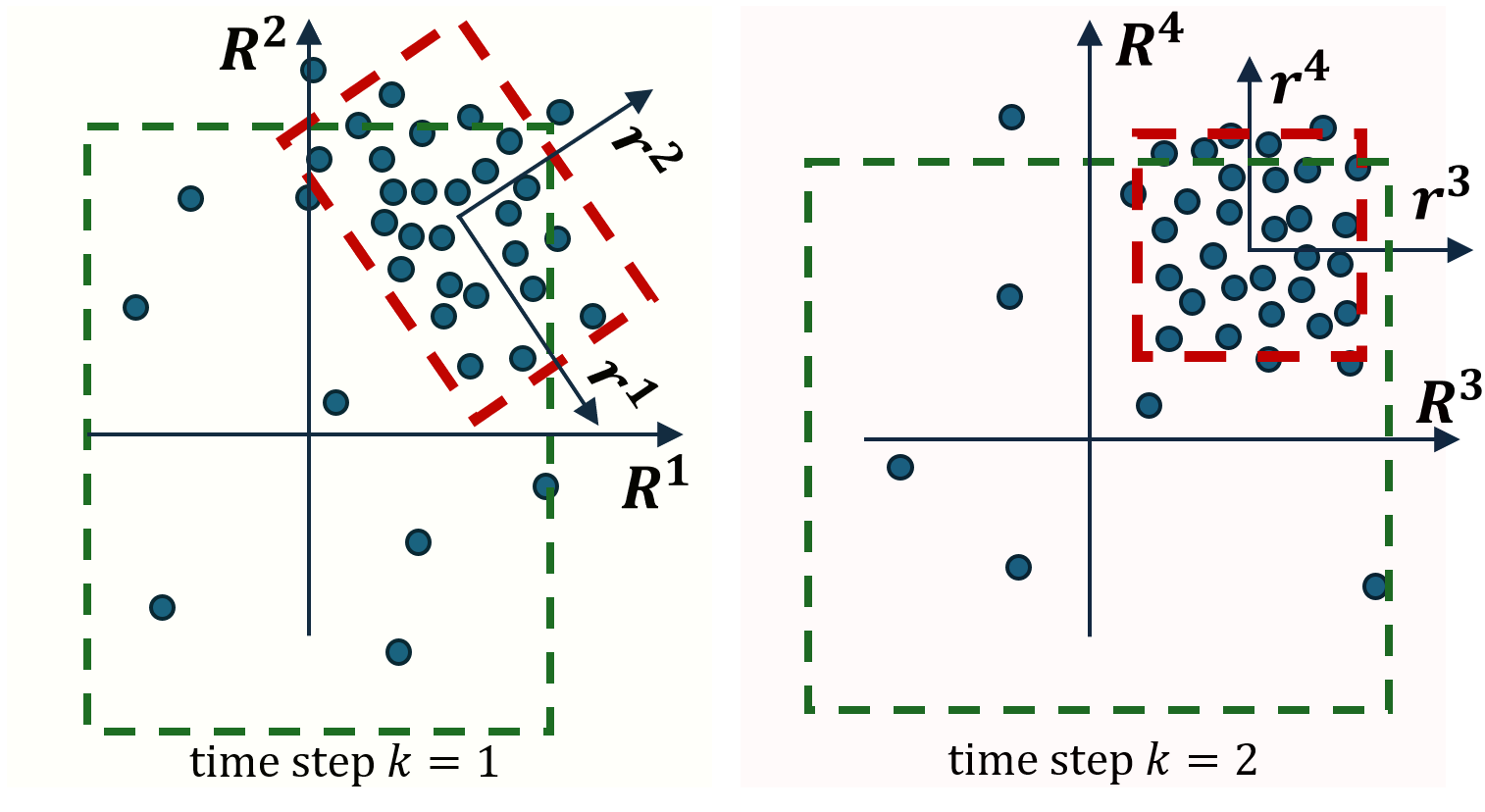}
    \vspace{-4mm}
    \caption{The figure shows the projection of prediction errors for two-dimensional states over a horizon of $\horizon = 2$. The left figure illustrates the projection on the $(R^1, R^2)$ axes (e.g., $k=1$), and the right figure displays the projection on the $(R^3, R^4)$ axes (e.g., $k=2$). This figure provides a comparison between the inflating hypercubes for a confidence level $\delta \in (0,1)$, generated by the PCA approach (red hypercubes) and the method proposed in \cite{hashemi2024statistical} (green hypercubes). It clearly demonstrates the superior accuracy of the PCA technique compared to the other method. The principal axes for $k=1,2$ are $(r^1, r^2)$ and $(r^3, r^4)$, respectively.}
    \label{fig:compy}
    \vspace{-4mm}
\end{figure}
\navidd{Principal Component Analysis (PCA) is a mathematical technique used to identify the principal directions of variation in a dataset. Given a dataset of $L$ data points, \( x_i \in \mathbb{R}^{n}, i\in[L]  \), PCA estimates the covariance matrix, $\Sigma \succeq 0, \Sigma \in \mathbb{R}^{n\times n}$ of data points $x_i, i\in[L]$. The eigenvectors of \( \Sigma \), known as \textbf{principal components}, define the directions along which the data exhibits the highest variance, while the corresponding eigenvalues quantify the magnitude of variance along each direction. These principal components form an orthonormal basis that aligns with the natural structure of the data, providing key insights into its intrinsic geometric properties.}

The authors in \cite{hashemi2024statistical} use the function in \eqref{eq:Rmax} to define the residual $\rho$ and compute the corresponding inflating hypercube $\delta X$. 
\navid{However, their technique imposes two conservative constraints. First, the center of the hypercube is always located at the origin. Second,
the edges of the hypercube are restricted to be aligned with the direction of the trajectory state components.}

To address these limitations, we propose \navid{an adjustment on the definition of the residual $\rho$ in equation \eqref{eq:Rmax},} that enables us to overcome these issues. We integrate the concepts of Conformal Inference (CI) and Principal Component Analysis (PCA) in our new definition for the residual. This approach provides the principal axes as the orientation of inflating hypercube and noticeably reduces its size. In other words, our approach enhances the accuracy of conformal inference by manipulating the coordinate system, inspired by PCA. However, in the context of CI, altering the coordinate system has also been addressed in other works, such as \cite{tumu2024multi,sharma2024pac}.

To obtain the principal axes, given the simulated trajectories from the training dataset $\trajsim_{\statee_{0,i}} \in \traindataset, i\in [|\traindataset|]$, for each segment $q\in[N]$, we use the trajectory segment $\trajsimseg{q}_{\statee_{0,i}}$ and its corresponding surrogate model $\overallf_q(\statee_{0,i}; \theta_q)$ to compute the corresponding set of prediction errors. Specifically, for each segment $q\in [N]$ and data index $i\in[|\traindataset|]$, we collect: 
\begin{equation}
\PEsimseg{q}_i = \left[ \errsim{t_q n+1}_i, \errsim{t_q n+2}_i, \ldots, \errsim{(t_q+T_q)n}_i \right]
\end{equation}
and approximate the average and covariance as follows:
\begin{equation}
\navid{\overline{\PE}^{q}} = \frac{\sum_{i=1}^{|\traindataset|} \PEsimseg{q}_i}{|\traindataset|} , \ \Sigma^q = \frac{\sum_{i=1}^{|\traindataset|} \left(\PEsimseg{q}_i \!\!- \overline{\PE}^{q}\right)^\top \left(\PEsimseg{q} \!\!- \overline{\PE}^{q}\right)}{|\traindataset|}.    
\end{equation}
We then apply spectral decomposition on the covariance matrix $\Sigma^q$ to obtain the array of eigenvectors $\eigvecseg{q} \in \mathbb{R}^{T_q n \times T_q n}$. Here, the principal axes for the trajectory segment $q \in [N]$ are centered on $\overline{\PE}^{q}$, and are aligned with the eigenvectors $\eigvecseg{q}_\ell, \ell \in [T_q n]$, which are the $\ell$-th columns of the matrix $\eigvecseg{q}$.

Given the initial state, $\statee_0 \sim \mathcal{W}$, assume a trajectory $\statee_1, \ldots, \statee_\horizon$ that is not necessarily sampled from $\distzero$ and also is not necessarily a member of the training dataset. For any segment $q \in [N]$ of this trajectory, we map its vector of prediction errors $\PE^q = \left[ R^{t_q n+1}, R^{t_q n+2}, \ldots, R^{(t_q+T_q)n} \right]$ to the principal axes. We do this with a linear map, as, 
\begin{equation}\label{eq:PCAerror}
\left[ r^{t_q n+1},r^{t_q n+2}, \ldots, r^{(t_q+T_q)n} \right] = \transpose{\eigvecseg{q}} (\PE^q - \overline{\PE}^{q}),
\end{equation} 
and utilize the parameters $r^j,\ t_q n+1\leq j \leq (t_q+T_q)n$ to define the residual. Collecting the mapped prediction errors for all segments $q\in[N]$, we propose our definition for residual as follows:
\begin{equation}\label{eq:PCAres}
    \rho:=  \max \left(\  \frac{|r^1|}{\omega_1} ,\  \frac{|r^2|}{\omega_2} ,\  \ldots \ , \  \frac{|r^{n\horizon}|}{\omega_{n\horizon}}\ \right)
\end{equation}
where the scaling factors $\omega_j , j \in [n\horizon]$ are the maximum magnitude of parameters $r^j_i , i\in |\traindataset|$, that are obtained from the training dataset. In other words,
\begin{equation}
\omega_j = \max( |r^{j}_1|,\  |r^{j}_2|,\  \ldots,\  |r^{j}_{|\traindataset|}|  ),\ \  j \in[n\horizon]. 
\end{equation}

Although we use the training dataset $\traindataset$ to determine hyperparameters, $\eigvecseg{q}$, $\overline{\PE}^q$, and $\omega_j$ for defining the residual, reusing $\traindataset$ to generate the inflating hypercube with robust conformal inference violates CI rules. Thus, in order to generate the inflating hypercube, we first sample a new i.i.d. set of trajectories from the training environment $\distzero$, which we denote as the calibration dataset.
\begin{definition}[Calibration Dataset] The calibration dataset $\calibdataset$ is defined as:
\begin{equation}
\calibdataset \!=\! 
\left\{ 
    \left(\statee_{0,i}, \rho_i \right) \middle|
    \begin{array}{l}
        \statee_{0,i}\sim \mathcal{W}, \, \trajsim_{\statee_{0,i}} \sim \distzero,  \\
        \rho_i = \max( \frac{|r^1_i|}{\omega_1}, \ldots, \frac{|r^{n\horizon}_i|}{\omega_{n\horizon}})
    \end{array}
\right\}.
\end{equation}
\noindent Here, $\trajsim_{\statee_{0,i}}, i \in |\calibdataset|$ refers to the trajectory starting at the $i^{th}$ initial state sampled from $\mathcal{W}$, generated from $\distzero$. The parameters $r^j_i$ are also as defined in equation~\eqref{eq:PCAerror}.
\end{definition}

Consider sorting the i.i.d. residuals $\rho_i \sim \distzeroR$ collected in the calibration dataset $\calibdataset$ by their magnitude: $\rho_1 < \rho_2 < \ldots < \rho_{|\calibdataset|}$. Our goal is to provide a provable upper bound for the $\delta$-quantile of a residual $\rho \sim \distR$, given knowledge of a radius $\tau > 0$ such that the total variation $\tv(\distR, \distzeroR) < \tau$. In this case, robust conformal inference \cite{cauchois2020robust} suggests using the rank $\ell^*$ from equation \eqref{eq:robustrank} and selecting $\rho^*_{\delta, \tau} := \rho_{\ell^*}$ as an upper bound for the residual's $\delta$-quantile. In other words, for a residual $\rho \sim \distR$, we have $\Pr[\rho < \rho^*_{\delta, \tau}] > \delta$.

\begin{proposition}\label{prop:maxcontribution}
Assume $\rho^*_{\delta,\tau}$ is the $\delta$-quantile of $\rho \sim \distR$, computed over the residuals $\rho_i \sim \distzeroR$ from the calibration dataset $\calibdataset$ where $\tv(\distR, \distzeroR)<\tau$. For the residual $\rho=  \max \left(\  \frac{|r^1|}{\omega_1} ,\  \frac{|r^2|}{\omega_2} ,\  \ldots \ , \  \frac{|r^{n\horizon}|}{\omega_{n\horizon}}\ \right)$ sampled from the distribution $\distR$, and the trajectory division setting, $T_q, q\in[N]$, it holds that, $\Pr\left[ P(r^1,\ldots,r^{n\horizon}) = \top \right] > \delta$, where,
\vspace{-3mm}
\begin{equation}\label{eq:probguarantee}
P(r^1,\ldots,r^{n\horizon}) \!\! = \!\!\! \bigwedge_{q=1}^{N} \!\! P_q(r^{t_qn+1}\!\!\!\!\!\!\!\!\!\!\!,\ldots,r^{(t_q+T_q)n}),\ \ P_q(r^{t_qn+1}\!\!\!\!\!\!\!\!\!,\ldots,r^{(t_q+T_q)n}) \!\! :=\!\!\!\!\!\!\!\!\bigwedge_{j=t_qn+1}^{(t_q+T_q)n} \!\!\!\!\!\!\!\left(-\omega_j\rho^*_{\delta,\tau}\leq r^j \leq \omega_j\rho^*_{\delta,\tau} \right) ,
\end{equation}
\vspace{-4mm}
and $r^j$ is the mapped version of prediction errors $R^j,\ j\in [n\horizon]$ on principal axes.
\end{proposition}
\begin{proof}
The proof follows as the residual $\rho$ is the maximum of the normalized version of parameters $r^j, j\in n\horizon$ so that
\vspace{-3mm}
\begin{equation}
\rho=  \max \left(\  \frac{|r^1|}{\omega_1} ,\  \frac{|r^2|}{\omega_2} ,\  \ldots \ , \  \frac{|r^{n\horizon}|}{\omega_{n\horizon}}\ \right)\!\! \iff \!\! \bigwedge_{j=1}^{n\horizon} \left[ |r^j| \leq \rho \omega_j \right].
\end{equation}
\vspace{-3mm}\\
Now, since $\Pr[ \rho \leq \rho^{*}_{\delta,\tau} ] \geq \delta$ as well as $\rho<\rho^*_{\delta,\tau} \iff |r^j|< \rho^*_{\delta,\tau} \omega_j$ for all $j\in [n\horizon]$, we can claim that  $\Pr[ \bigwedge_{j=1}^{n\horizon} [ |r^j| \leq \rho^*_{\delta,\tau} \omega_j ]] \geq \delta$. The guarantee proposed in \eqref{eq:probguarantee} is the reformulation of this results in terms of the division setting, $T_q, q\in [N]$.    
\end{proof}
Referring to Def.~\ref{def:star}, and using the predicates $P_q(r^{t_qn+1}\!\!\!\!\!\!\!\!\!,\ldots,r^{(t_q+T_q)n}), q\in[N]$ from Proposition \ref{prop:maxcontribution} we can introduce the inflating hypercubes, $\delta X_q , q\in [N]$ as star sets. In other words, from equation \eqref{eq:PCAerror}, for any $q\in[N]$ we can compute the prediction errors as,
\begin{equation}
\PE^q = \overline{\PE}^q + \eigvecseg{q}\left[ r^{t_q n+1},r^{t_q n+2}, \ldots, r^{(t_q+T_q)n} \right]
\end{equation}
which implies $\delta X_q = \langle \overline{\PE}^q, \eigvecseg{q}, P_q(r^{t_qn+1}\!\!\!\!\!\!\!\!\!,\ldots,r^{(t_q+T_q)n}) \rangle$, and thus the concatenation of the star sets $\delta X_q$, serves as an inflating hypercube for $\PE$. \navid{Therefore, we denote the inflating hypercube of the entire trajectory by $\delta X = \langle \overline{\PE}, V, P \rangle$ where $V = \mathbf{diag}\left(V^1, \ldots, V^N\right)$ and $P = \bigwedge_{q=1}^N P_q$.} 

\navid{Finally, based on Lemma \ref{lem:inclusion_conformal}, the $\delta$-confident flowpipe on the entire \navidg{trajectory} $X$ can be obtained through the inflation of surrogate flowpipe $\bar{X}$ with the inflating hypercube $\delta X$, i.e., $X = \bar{X}\oplus \delta X$ }.


\vspace{-3mm}
\begin{remark}
Figure \ref{fig:compy} shows the advantage of the PCA approach by illustrating prediction errors of a $2$-dimensional state over $2$ consecutive time steps\footnote{Division setting: $\horizon =2,\  n=2,\  N=2$, and $T_1=T_2= 1$.}. This figure also provides a schematic of the inflating hypercubes generated by our residual definition and those generated by the definition proposed in \eqref{eq:Rmax}. The primary function of the vectors $\overline{\PE}^q, q\in[N]$ is to reposition the surrogate reachsets $\bar{X}_q$ to locations that require minimal inflation, and the main role of $\eigvecseg{q}$ is to further reduce the necessary level of inflation.
\end{remark}

\vspace{-6mm}
\section{Numerical Evaluation}
\label{sec:Experiments}

To simulate real-world systems capable of producing actual trajectory data, we employ stochastic difference equation-based models with additive Gaussian noise to account for uncertainties in observations, dynamics, and potential modeling errors. Our theoretical guarantees apply to any real-world distribution $\trajreal_{\statee_0} \in \dist$, provided that the residual distribution shift $\tv(\distzeroR, \distR)$ is below a given threshold $\tau$. Here we evaluate our results on three different case studies. The first two experiments involve a $12$-D quadcopter with $\tau = 0$, while the final experiment focuses on a $27$-D powertrain model where the distribution shift is upper-bounded at $\tau = 4\%$. In Experiment 1, we compare our approach with \cite{hashemi2024statistical}. However, since that methodology does not scale well for trajectories with large number of time-steps, $\horizon$, we address Experiments 2 and 3, only using the methodology proposed in this paper. The next two sections provide a general overview of the experiments, with detailed information deferred to the Appendix. Table \ref{tbl:expdetails} also presents the details of the numerical results.

\begin{table*}
\hspace{1mm}
\resizebox{0.85\textwidth}{!}{
\begin{tabular*}{\textwidth}{@{\extracolsep{\fill}}cccccccccc}
\cmidrule{1-10}
\muc{1}{} & \muc{2}{Specification} & \muc{3}{Training} & \muc{2}{Surrogate Reachability} & \muc{2}{Inflating Hypercube}\\
\cmidrule{2-3}\cmidrule{4-6}\cmidrule{7-8} \cmidrule{9-10}
    Exp $\#$: & $\delta$ & $\tau$  & $\#$ & avg runtime  & $\mid \traindataset \mid$  &  $\#$ & avg runtime(method)  & runtime & $\mid \calibdataset \mid$  \\
\cmidrule{1-10}
1    & $99.99\%$   & $0$    & $100$ & $39.6\ \sec$  &  $42,000$ & $100$  & $1.43\ \sec\ $ (E)   & $2.08\ \sec$ & $20,000$\\
2    & $99.99\%$   & $0$    & $451$ & $33.65\ \sec$  &  $20,000$ & $4501$ & $0.030\ \sec\ $ (E)  & $116.58\ \sec$ & $20,000$\\
3    & $95\%$      & $4\%$  & $400$ & $40.6\ \sec$  &  $10,000$ & $4000$ & $0.064\ \sec\ $ (A)  & $142.02\ \sec$ & $10,000$\\
\cmidrule{1-10}
\end{tabular*}
}
\vspace{-3mm}\\
\caption{Shows details of the experiments. The models are trained in parallel with $18$ CPU workers. Thus, the average training runtime may vary by selecting different number of workers. The words E, and A represent exact-star and approx-star, respectively.  }
\vspace{-6mm}
\label{tbl:expdetails}
\end{table*}
\vspace{-2mm}
\subsection{$12$-Dimensional Quadcopter}
\vspace{-2mm}
We consider the $12$-dimensional quadcopter system under stochastic conditions for two different case studies. Trajectories are simulated using two ODE models from \cite{hashemi2024statistical} and \cite{hashemi2024scaling} as our simulators. The state variables include the quadcopter's position $(x_1, x_2, x_3)$, velocity $(x_4, x_5, x_6)$, Euler angles $(x_7, x_8, x_9)$ representing roll, pitch, and yaw angles, and angular velocities $(x_{10}, x_{11}, x_{12})$. We also include zero mean additive Gaussian process noise $v \sim \gaussian(0_{12\times1}, \Sigma_v)$ to the simulators with covariance $\Sigma_v = \mathbf{diag}\left( [0.05 \times \vec{1}_{1\times6} ,\  0.01\times \vec{1}_{1\times6]}]^2 \right)$.  In both examples, the set of initial states $\statee_0 \in \init$ is taken from the cited papers, with the distribution $\statee_0 \sim \mathcal{W}$ being uniform.
\vspace{-2mm}
\subsection{27-Dimensional Powertrain }
\vspace{-2mm}
We use the powertrain system proposed by \cite{althoff2012avoiding} as our simulator, which is a hybrid system with three modes. To introduce stochastic conditions, we add zero-mean Gaussian process noise, $v \sim \gaussian(\vec{0}_{27\times1}, \Sigma_v)$, where $\Sigma_v = \mathbf{diag}\left( 10^{-5}\times \vec{1}_{1\times27} \right)$, to their simulator, defining the distribution $\trajsim_{\statee_0}\sim \distzero$.
This system is highly sensitive to noise, which is a key reason we addressed it in this paper. For example, Figure \ref{fig:noisecontribution} shows the angular velocity of the last rotating mass, $x_{27}$, both with and without noise. Following \cite{althoff2012avoiding}, we simulate trajectories with a sampling time of $\delta t = 0.0005$ over a horizon of $2$ seconds ($\horizon = 4000$), and consider their set of initial states $\init$ \footnote{In this case the set $\init$ proposed in \cite{althoff2012avoiding} is a large and high dimensional set, thus the exact star does not scale, and we are restricted to utilized approx star for surrogate reachability.}. We also define the trajectory division setting as $N=4000$, $T_q=1, q\in[N]$. The $\relu$ NN models are with structure $[27, 54,27]$. To reduce the training runtime, we again follow the analytical interpolation strategy we introduced for Experiment 2.

\section{Acknowledgements}
This work was partially supported by the National Science Foundation through the following grants: CAREER award (SHF-2048094), CNS-1932620, CNS-2039087, FMitF-1837131, CCF-SHF-1932620, IIS-SLES-2417075, funding by Toyota R\&D and Siemens Corporate Research through the USC Center for Autonomy and AI, an Amazon Faculty Research Award, and the Airbus Institute for Engineering Research. This work does not reflect the views or positions of any organization listed.

\section{Conclusion}
We introduced a scalable technique for reachability in real-world settings. Our results demonstrate that integrating PCA with Conformal inference significantly enhances the accuracy of error analysis. We validated the effectiveness of our approach across three distinct high-dimensional environments.



\appendix

\begin{figure*}
\centering
\includegraphics[trim={3.6cm
  1cm 3.7cm 1cm},width =0.85\linewidth]{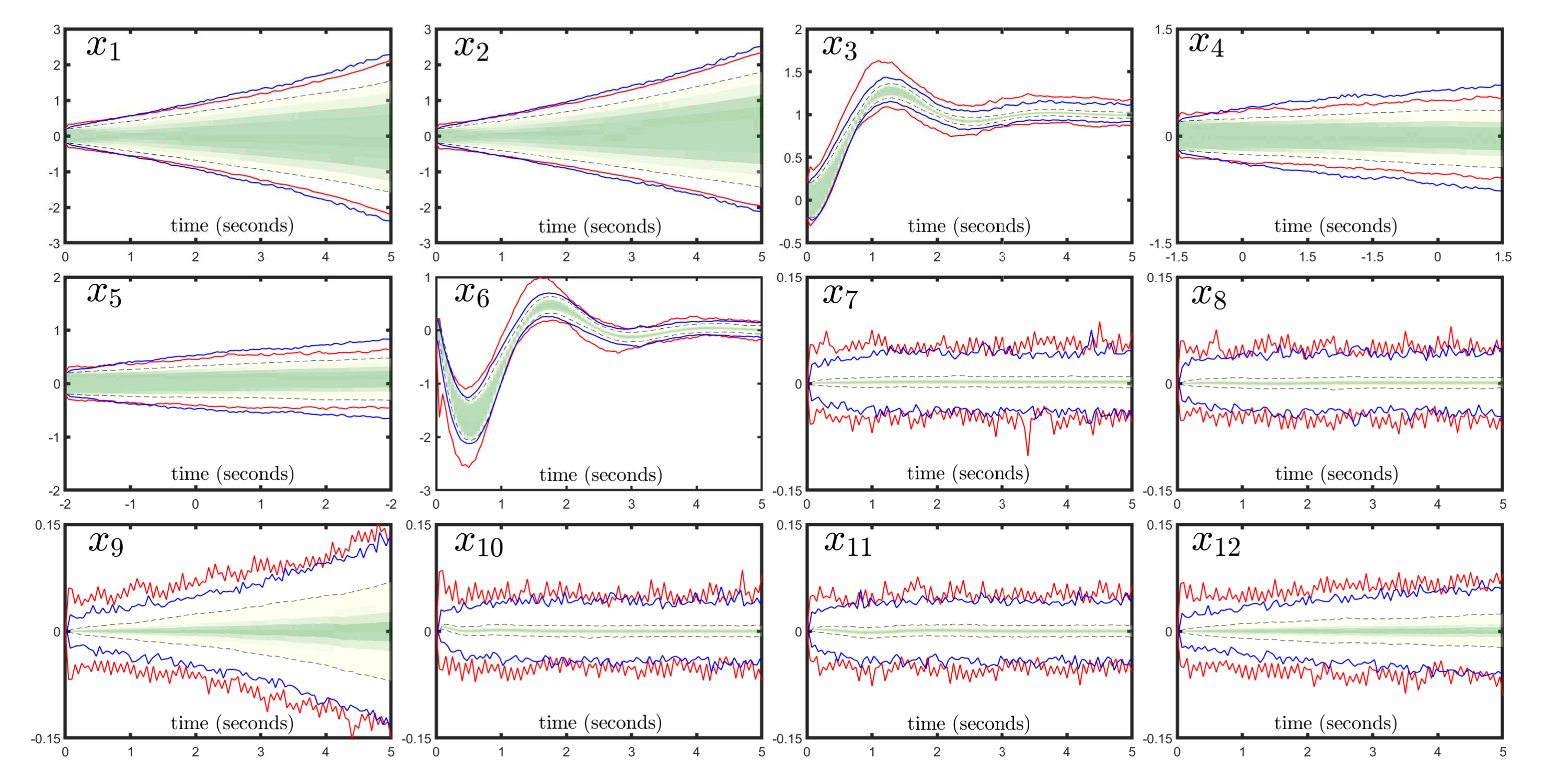}
\caption{Shows the comparison with \cite{hashemi2024statistical}. The blue and red borders are projections of our and their $\delta$-confident flowpipes respectively with $\delta = 99.99\%$. The shaded regions show the density of the trajectories from $\traindataset$.}
\label{fig:compEmsoft}
\end{figure*}
\begin{figure*}
\centering
\includegraphics[trim={3.6cm
  1cm 3.4cm 1cm},width =0.85\linewidth]{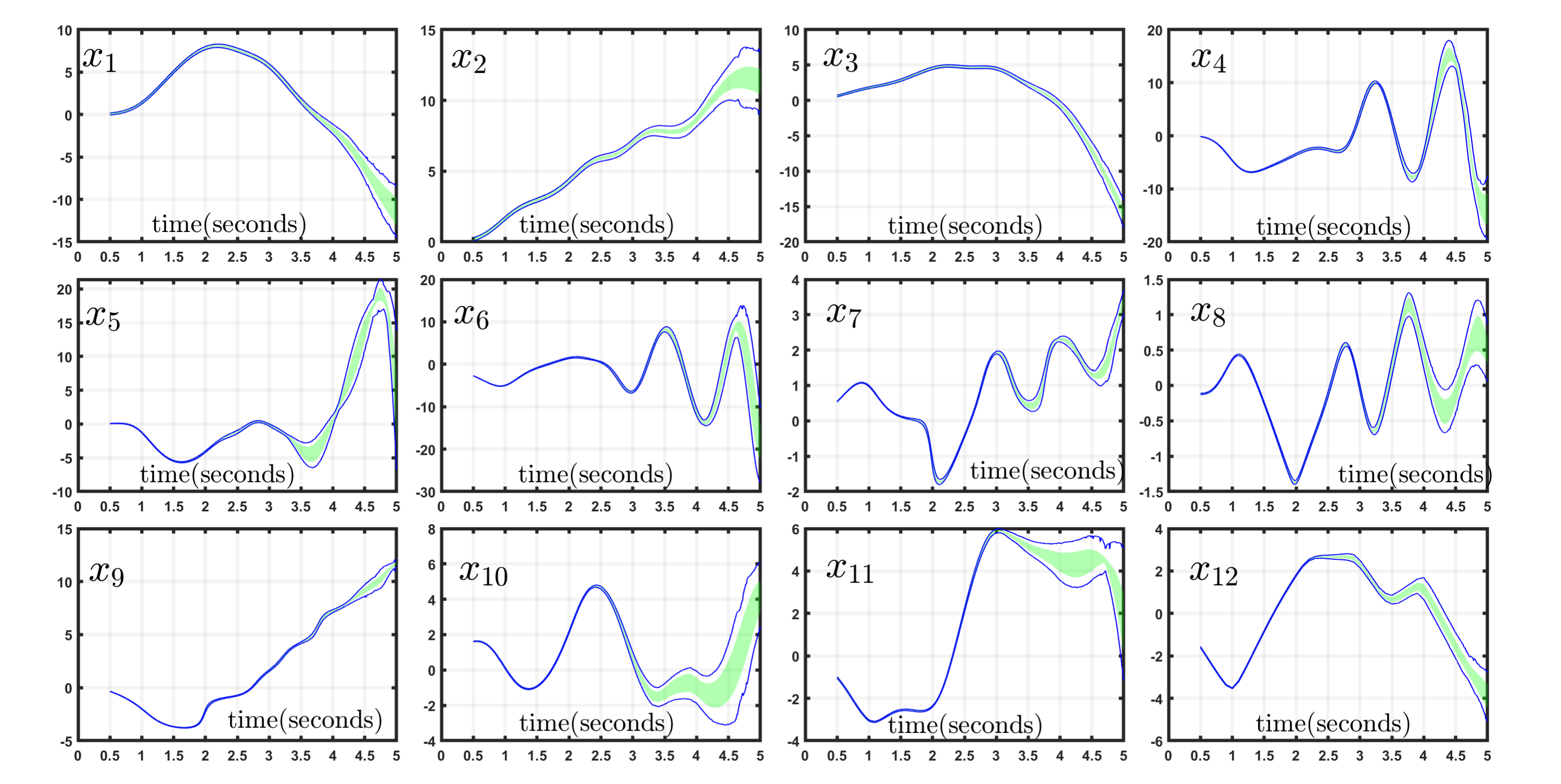}
\caption{ Shows the projection of our $\delta$-confident flowpipe on each component of the trajectory state. The shaded area are the simulation of trajectories from $\traindataset$.}
\label{fig:nested}
\vspace{-5mm}
\end{figure*}

\begin{figure*}
    \begin{minipage}[t]{0.76\textwidth}
        \centering
        \includegraphics[trim={3.2cm 0.5cm 3.4cm 2cm},width = 0.95\textwidth]{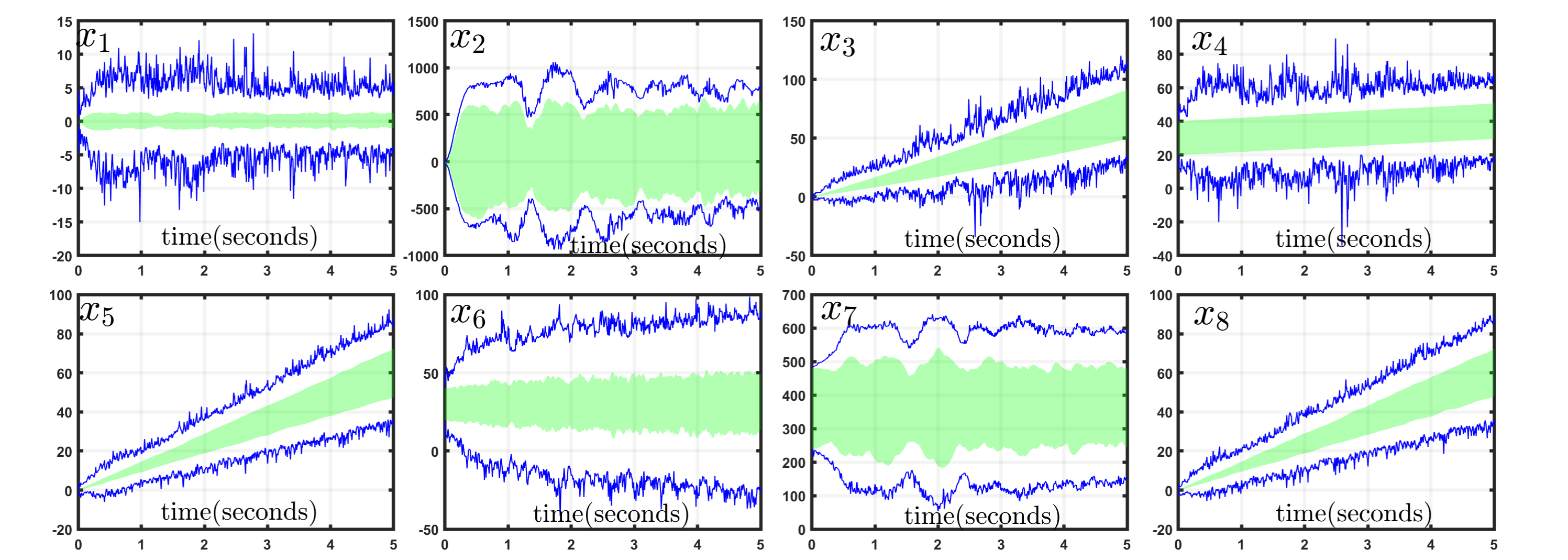}
        \caption{Shows the projection of our $\delta$-confident flowpipe on the first $8$ components of the trajectory state. There is a shift between the distribution of deployment and training environments. The shaded area are the trajectories sampled from the deployment environment.}\label{fig:proj}
    \end{minipage}
    \hfill
    \begin{minipage}[t]{0.21\textwidth}
        \vspace{-4.5cm}
        \includegraphics[width=\textwidth]{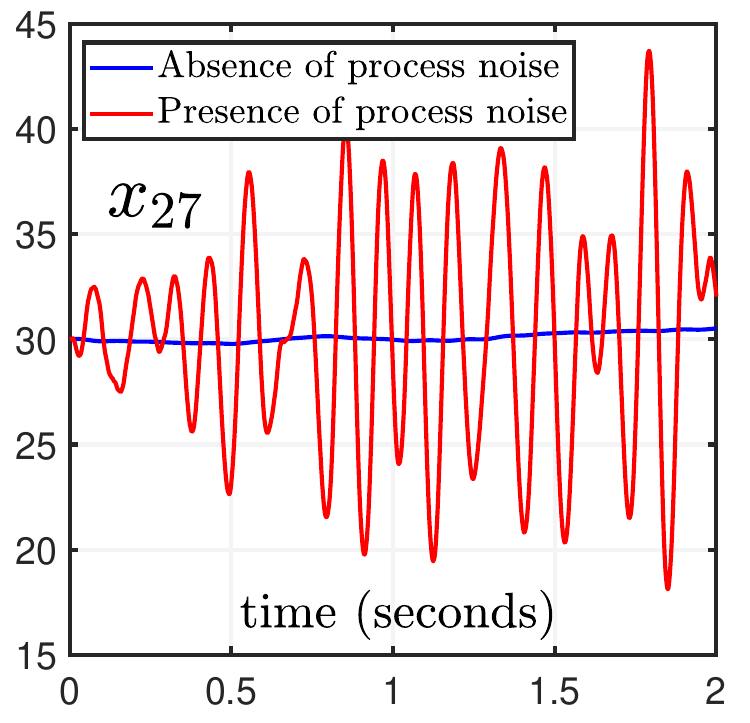}
        \caption{Shows the comparison of angular velocity of the last rotating mass in presence and absence of the process noise.}
        \label{fig:noisecontribution}
    \end{minipage}
    \vspace{-5mm}
\end{figure*}
\section{Detail of the Experiments}
\subsection{Experiment 1:[Comparison with \cite{hashemi2024statistical}]} Here we address Experiment 2 from \cite{hashemi2024statistical} for comparison of the results. In this experiment, a quadcopter hovers at a specific elevation, and its trajectories are simulated over a horizon of $\horizon = 100$ time steps, with a sampling time of $\delta t = 0.05$. The $\delta$-confident flowpipe has a confidence level of $\delta = 99.99\%$. Compared to \cite{hashemi2024statistical}, our approach achieves a higher level of accuracy. This improvement is due to our training strategy, which allows us to use exact-star for surrogate reachability, and our PCA-based technique, which results in smaller inflating hypercubes. In this experiment, we use a trajectory division setting of $N=100$, $T_q=1$, for $q\in[N]$, with $\relu$ neural network surrogate models structured as $[12, 24, 12]$. Figure \ref{fig:compEmsoft} shows the projection of the flowpipe on each state in comparison with the results of \cite{hashemi2024statistical}, and Table \ref{tbl:expdetails} shows the detail of the experiment.

\subsection{Experiment 2: [Sequential Goal Reaching Task]} In this example, we consider the quadcopter scenario described in \cite{hashemi2024scaling}, where a controller is designed to ensure that the machine accomplishes a sequential goal-reaching task. We also include the previously mentioned process noise in the simulator to include stochasticity. Given the quadcopter's tendency for unpredictable behavior, we significantly reduce the sampling time in this instance. The trajectories are sampled at a frequency of 1 KHz over a 5-second horizon, resulting in 5000 time steps. Our objective is to perform reachability analysis for time steps 500 through 5000 with the level of confidence $\delta = 99.99\%$. We propose a trajectory division setting of $N=5000$ with $T_q = 1$ for $q \in [N]$. To reduce the runtime for model training, we employ analytical interpolation. Specifically, we select every tenth time step for model training, and for $i \in 50, 51, \ldots, 500$, and $j \in [10]$, we regenerate all the $\relu$ neural network surrogate models using the following formula:
\begin{equation}\label{eq:strategy}
\overallf_{10i+j} = (1-0.1j)\overallf_{10i} + 0.1j\overallf_{10(i+1)} 
\end{equation}
where the models $\overallf_{10i}$ have a structure of $[12, 24, 12]$. We then utilize these regenerated $\relu$ neural network surrogate models for surrogate reachability through exact star reachability analysis, as well as error analysis using PCA-based conformal inference. Figure \ref{fig:nested} shows the resulting flowpipe and Table \ref{tbl:expdetails} shows the detail of the experiment.

\subsection{Experiment 3: [Reachability with Distribution shift]} Let's assume the real world trajectories $\trajreal_{\statee_0} \sim \dist$ are such that its covariance of process noise is $20\%$ larger than $\Sigma_v$. In this case, the threshold $\tau = 0.04$ is a valid upper-bound for $\tv(\distzeroR, \distR)$. In this experiment, given the threshold, $\tau$ we generate a $\delta$-confident flowpipe for $\trajreal_{\statee_0}$ with $\delta = 95\%$. Figure \ref{fig:proj} shows the projection of our computed flowpipe on the first $8$ components of states, and Table \ref{tbl:expdetails} shows the detail of the experiment.

\end{document}